\documentclass[10pt,twocolumn,letterpaper]{article}

\usepackage[pagenumbers]{cvpr} 

\usepackage{microtype}

\DeclareMathOperator{\sg}{sg}

\DeclareMathOperator{\Pool}{Pool}
\newcommand{\E}{\mathbb{E}}
\newcommand{\KL}{\mathrm{D}_{\mathrm{KL}}}

\usepackage{xspace}
\newcommand{\method}{DyMD\xspace}
\usepackage{algorithm}
\usepackage{algpseudocode}
\usepackage{amsthm}
\usepackage{multirow}
\usepackage{colortbl}
\definecolor{intdrei}{RGB}{214,230,244}
\newtheorem{proposition}{Proposition}

\makeatletter
\@ifpackageloaded{lineno}{%
  \let\ptdmd@LN@legacy@makecol\@LN@makecol
  \RemoveFromHook{build/column/before}[lineno]
  \AtBeginDocument{\let\@LN@makecol\ptdmd@LN@legacy@makecol}%
}{}
\makeatother

\definecolor{cvprblue}{rgb}{0.21,0.49,0.74}
\usepackage[pagebackref,breaklinks,colorlinks,allcolors=cvprblue]{hyperref}

\def\paperID{*****} 
\def\confName{CVPR}
\def\confYear{2027}

\title{\method: Preserving Interaction Dynamics through Distribution Matching Distillation in Few-Step Video World Models}

\author{
Haojun Xu$^{1,2}$ \quad
Jie Huang$^{2,\ddagger}$ \quad
Xin Lu$^{2}$ \quad
Mingchen Zhong$^{2}$ \quad
Zihao Fan$^{2}$ \quad
Linjiang Huang$^{1,\dagger}$ \quad
Si Liu$^{1}$\\[3pt]
$^{1}$Beihang University
\quad
$^{2}$JD Future Academy\\
{$^{*}$Equal contribution.
\quad
$^{\dagger}$Corresponding author.
\quad
$^{\ddagger}$Project leader.}}

\begin{document}
\maketitle
\begin{abstract}
Large video diffusion models provide rich priors for embodied prediction and learning. However, generating visual futures with these models requires costly iterative sampling.
Distribution Matching Distillation (DMD) enables few-step video generation,
but suppress robot--object motion while preserving visual quality. Examining DMD's teacher and fake-score signals, we find that weak re-noising keeps the teacher posterior concentrated
near motion-deficient rollouts, limiting motion-restoring guidance. Meanwhile, stronger-motion rollouts tend to incur larger fake-score fitting errors, which hinder
the generator's learning of interaction dynamics.
We propose \method, a DMD framework that adapts both teacher supervision
and critic fitting to the evolving student.
Temporal affinity--conditioned re-noise sampling adapts the timestep
distribution to each rollout's current interaction fidelity by mixing
the base schedule with a teacher prior motivated by local posterior
variation, thereby balancing motion recovery and appearance refinement.
To better track stronger-motion rollouts which are empirically associated with higher flow-matching loss, dynamics-guided fake-score
tracking uses a noise-conditioned predictor to estimate noise-relative
fitting difficulty, then upweights
predicted-hard rollouts in the critic loss.
Using \method, we distill a 14B teacher into a four-step 1.3B student with no auxiliary
modules at inference. On embodied-video benchmarks, the student improves
R-Bench task adherence by $9.6$ percentage points and PAI-Bench-G Domain
score by $5.1$ points over Base DMD while maintaining comparable visual
quality. As a backbone for downstream action planning, our student
achieves $34\%$ mean success across two WorldArena tasks, compared with
$16\%$ for Base DMD.
\end{abstract}

\section{Introduction}
\label{sec:intro}

\begin{figure}[t]
  \centering
  \IfFileExists{fig/teaser.pdf}{%
    \includegraphics[page=1,width=\columnwidth,height=0.28\textheight,
      keepaspectratio]{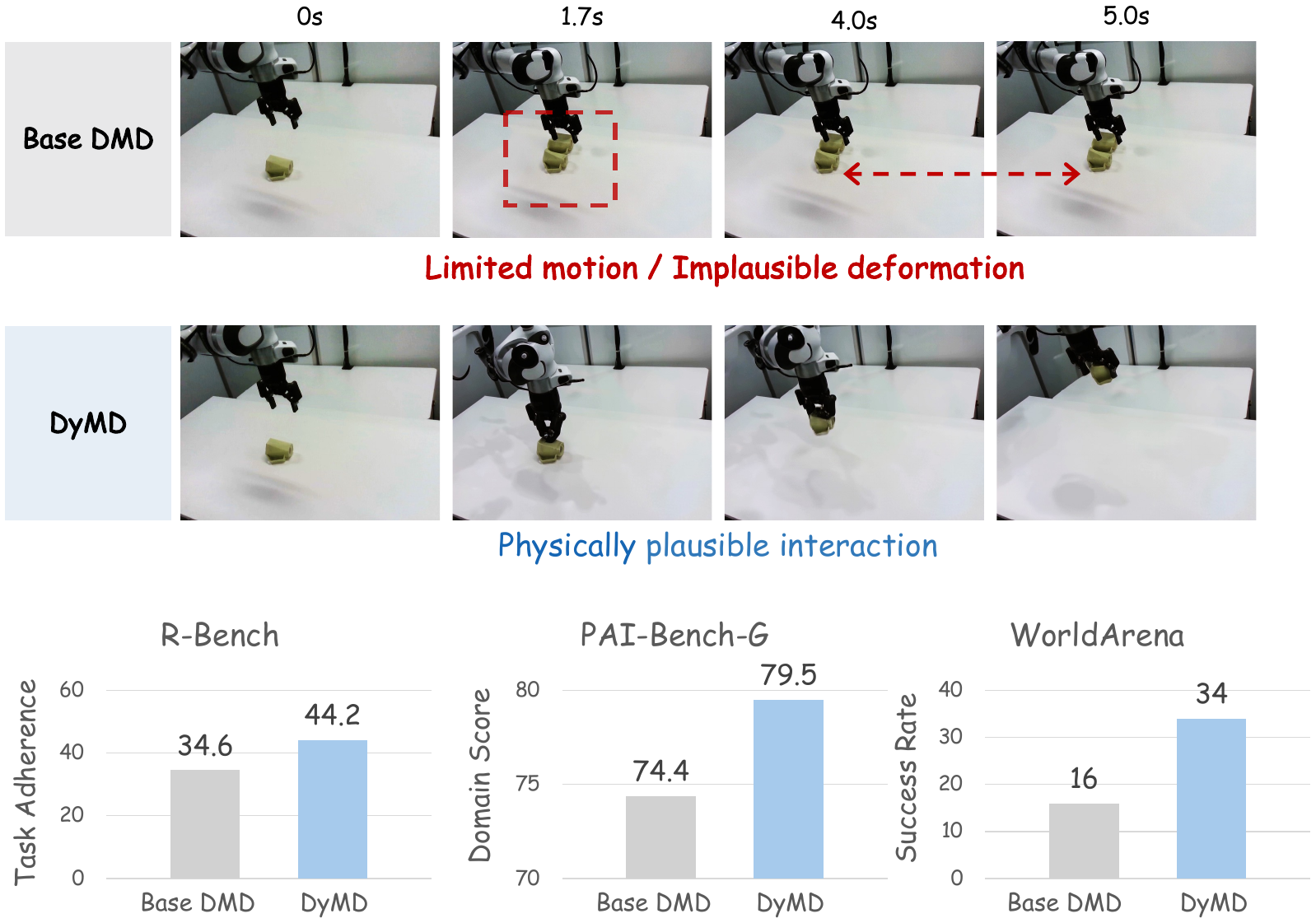}%
  }{%
    \fbox{\parbox[c][0.28\textheight][c]{\dimexpr\columnwidth-2\fboxsep-2\fboxrule\relax}{%
      \centering\large\color{gray}Teaser figure placeholder\\[1em]
      \normalsize Video comparison\\and benchmark results}}%
  }
  \caption{\textbf{Preserving interaction dynamics in few-step video world models.}
  Top: Base DMD shows implausible deformation and limited motion; \method
  preserves a physically plausible robot--object interaction. Bottom: gains in
  R-Bench task adherence, PAI-Bench-G domain score, and WorldArena planning
  success. Both are 1.3B four-step students.}
  \label{fig:teaser}
\end{figure}

\begin{figure*}[t]
  \centering
  \begin{minipage}[t]{\columnwidth}
    \vspace{0pt}
    \centering
    \includegraphics[width=\linewidth]{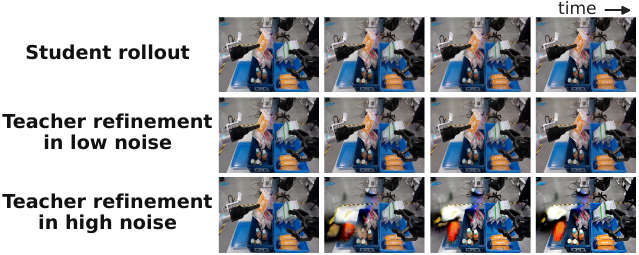}
    \caption{\textbf{Motion recovery depends on the re-noising level.}
    We re-noise the same nearly static four-step DMD rollout at different noise
    levels and refine it with the frozen PF-Wan teacher. Low re-noising levels
    leave the video nearly static, while higher levels recover motion but introduce visible artifacts. Four frames are shown in temporal order.}
    \label{fig:teacher-recovery}
  \end{minipage}\hfill
  \begin{minipage}[t]{\columnwidth}
    \vspace{0pt}
    \centering
    \includegraphics[width=\linewidth]{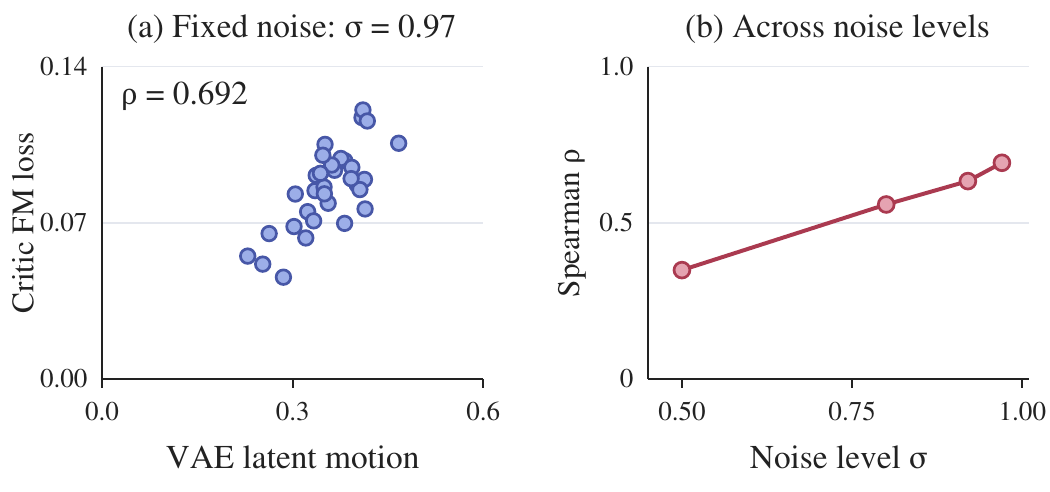}
    \caption{\textbf{Stronger latent motion is associated with higher critic FM loss.}
    Left: 32 frozen rollouts at $\sigma=0.97$.
    Right: Spearman correlation between VAE latent motion magnitude and
    critic FM loss across the same rollouts, computed separately at four
    noise levels. Protocol in Appendix~\ref{app:critic-predictability}.}
    \label{fig:motion-tracking-probe}
  \end{minipage}
\end{figure*}

Video world models support embodied prediction and learning by generating visual futures and providing representations for downstream policies~\cite{rovidx,dreamgen,du2023learning}.
For robotic manipulation, useful predictions must capture coordinated robot--object motion as well as scene appearance. However, generating such visual futures with large video diffusion models requires costly iterative sampling~\cite{wan,cosmos}, motivating distillation into few-step generators. Distribution Matching Distillation
(DMD)~\cite{dmd,dmd2} trains a few-step student using the score difference
between a frozen teacher and an online fake-score model (the critic) that
estimates the score of the evolving student distribution. But the challenge is
to retain task-relevant interaction dynamics under this reduced sampling
budget.

Our starting observation is that DMD preserve visual quality
while suppressing interaction dynamics. When distilling PF-Wan, a teacher
fine-tuned for robotic manipulation~\cite{physisforcing}, vanilla DMD
produces sharp videos with weak robot--object motion (Fig.~\ref{fig:teaser}).
To understand this failure, we examine the two components of DMD's score-difference update: the teacher guidance and the critic's estimate of the evolving student score.

First, the teacher's ability to correct this motion deficit depends on
the re-noising level. Weak re-noising keeps the teacher's clean-video
posterior concentrated near a motion-deficient student rollout, providing
little corrective guidance. Stronger re-noising of the same rollout
restores robot--object motion but may introduce artifacts
(Fig.~\ref{fig:teacher-recovery}). Typical schedules use the same timestep
distribution for all rollouts, regardless of their current interaction
fidelity. This motivates adapting re-noise sampling to balance motion
recovery and appearance refinement for each rollout.

Second, preserving interaction dynamics requires the critic to accurately
track the evolving student distribution, since errors in its student-score
estimate perturb the generator's score-difference update~\cite{dmd2,rtdmd}.
At fixed noise levels, our frozen-snapshot probe finds that rollouts with
stronger motion in VAE latent space tend to incur higher flow-matching
(FM) loss (Fig.~\ref{fig:motion-tracking-probe}), suggesting that these
rollouts are harder for the critic to fit. Thus, inaccurate critic scores on these rollouts hinder the generator from learning interaction dynamics effectively. To better track these
stronger-motion rollouts, we therefore use fitting difficulty to allocate
the critic's limited training effort.

These two challenges motivate \method,
which couples rollout-adaptive teacher supervision with
difficulty-aware critic fitting. To balance motion recovery and
appearance refinement, temporal affinity--conditioned re-noise
sampling adapts each rollout's mixture of a teacher prior and the
base schedule. The prior emphasizes pronounced changes in the
teacher's denoising direction, motivated by the connection between
teacher-velocity turning and local posterior variation. Temporal
affinity to a paired training video, measured in frozen V-JEPA
features~\cite{vjepa21}, assesses the rollout's current interaction
fidelity and controls this mixture: low-affinity rollouts receive
more teacher-prior sampling, while well-matched rollouts retain
more base-schedule coverage. To improve the critic's score estimates for rollouts with stronger motion, dynamics-guided fake-score tracking uses a noise-conditioned
predictor to estimate noise-relative fitting difficulty from
temporal dynamics in VAE latents. Rollouts predicted to be harder than typical samples at the same noise level receive larger
weights in the critic loss, without increasing the number of critic updates.

We distill the PF-Wan teacher into a four-step 1.3B student. Both components
operate only during training, with no auxiliary modules required at inference. Experiments on R-Bench~\cite{rovidx},
PAI-Bench-G~\cite{zhou2026pai}, and EZS-Bench~\cite{abotphysworld} show improved
interaction fidelity over Base DMD: R-Bench task adherence rises by $9.6$
percentage points and PAI-Bench-G Domain score by $5.1$ points, with
comparable visual quality. As a backbone for downstream action planning,
our student achieves $34\%$ mean success across two WorldArena
tasks~\cite{worldarena}, compared with $16\%$ for Base DMD.

Our contributions are threefold:
\begin{itemize}
  \item We analyze motion degradation in DMD through its teacher and
  fake-score signals, identifying teacher-side mode locking and
  motion-related critic fitting difficulty as obstacles to preserving
  interaction dynamics.

  \item We propose \method, combining temporal affinity--conditioned
  re-noise sampling with dynamics-guided fake-score tracking to adapt
  teacher supervision and critic fitting to the current student
  (Sections~\ref{sec:curvature}--\ref{sec:critic-reweight}).

  \item We demonstrate improved interaction fidelity over Base DMD on
  three embodied-video benchmarks and higher mean success on two
  downstream action-planning tasks with a four-step 1.3B student
  (Section~\ref{sec:experiments}).
\end{itemize}

\section{Related Work}
\label{sec:related}

\paragraph{Video World Models.}
Large-scale video diffusion and flow models learn expressive priors for
visual generation~\cite{wan,hunyuanvideo,cosmos,sora}. Robot video corpora and
video-based world models extend these priors to embodied prediction and
learning~\cite{rovidx,dreamgen,genie,du2023learning}, where object persistence, contact
consistency, and coordinated robot--object motion are essential.
VideoJAM jointly models appearance and motion to improve temporal
coherence~\cite{videojam}. ABot-PhysWorld uses physics-aware preference
optimization~\cite{abotphysworld}, while PhysisForcing introduces
region-focused physics supervision for robotic
interactions~\cite{physisforcing}. These advances improve the dynamics
available in a video teacher. Our focus is transferring those dynamics to a
few-step student through adaptive timestep sampling and critic fitting.

\paragraph{Diffusion Model Distillation.}
Few-step distillation transfers pretrained diffusion and flow models through
trajectory consistency~\cite{consistencymodel,lcm} or distribution-level
supervision, including adversarial learning~\cite{sdxlturbo} and score-based
distribution matching~\cite{dmd,dmd2}. Several extensions improve the supervision used by DMD. Reward-guided
methods combine distribution matching with preference
optimization~\cite{rewardforcing,dmdr}, while Adaptive Matching Distillation
reshapes score guidance to favor higher-quality samples~\cite{amd}.
For temporal fidelity, Adaptive Video Distillation combines adaptive paired
regression with temporal regularization~\cite{avd}. In generative video
compression, Adaptive Score Distillation attenuates updates that disagree
with a ground-truth reconstruction direction~\cite{asd}.
CoDMD adds a relational objective over samples and frames, reusing teacher
and fake-score predictions to preserve their dependency
structure~\cite{zhang2026codmd}.

\paragraph{Timestep Sampling and Fake-Score Tracking.}
Noise allocation and fake-score accuracy offer another axis of improvement.
Improved DDPM uses loss-aware timestep importance sampling to reduce gradient
noise~\cite{nichol2021improved}, while Min-SNR balances timestep losses through
SNR-based weighting~\cite{hang2023efficient}.
Phased DMD progressively matches intermediate distributions and derives
score-matching objectives within SNR subintervals for specialized
experts~\cite{fan2026phased}. DMD2 stabilizes fake-score estimation through
two-time-scale updates~\cite{dmd2}; RTDMD's ambient-consistent stage adds a
cross-timestep consistency regularizer to help fake-score model track
the generator distribution under limited updates~\cite{rtdmd}.
Learning Loss predicts sample losses for active data acquisition~\cite{yoo2019learning}, providing a related perspective on difficulty-based allocation.  Our approach uses temporal affinity to adapt each student rollout's re-noise schedule and a noise-conditioned predictor estimates noise-relative
fitting difficulty from latent temporal dynamics to reweight critic samples.
\section{Method}
\label{sec:method}

We first review DMD's generator and fake-score updates and the challenges
they pose for preserving interaction dynamics
(Section~\ref{sec:dmd-opd}).
Building on this analysis, \method (Fig.~\ref{fig:pipeline})
constructs a teacher-velocity-turning prior and uses temporal affinity
to adapt its mixture with the base schedule to each rollout's current
interaction fidelity
(Sections~\ref{sec:curvature} and~\ref{sec:router}).
It further predicts noise-relative fitting difficulty from latent
temporal dynamics to reweight critic samples within a fixed update
budget (Section~\ref{sec:critic-reweight}).

\begin{figure*}[t]
  \centering
  \includegraphics[width=\textwidth,trim=5bp 13bp 3bp 20bp,clip]{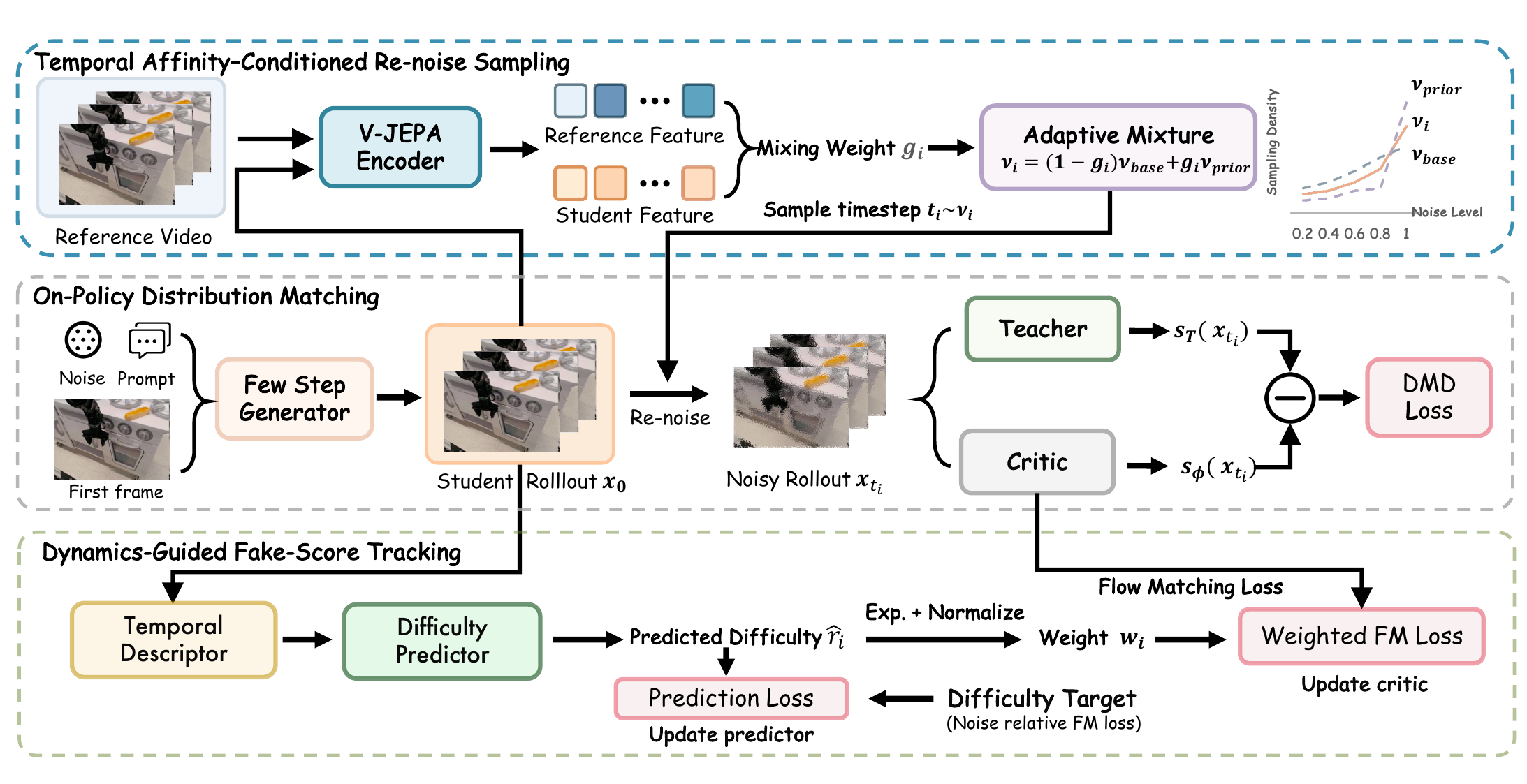}
  \caption{\textbf{Overview of \method.}
  Top (Secs.~\ref{sec:curvature}--\ref{sec:router}): temporal affinity controls
  each rollout's mixture of the teacher prior and base schedule.
  Middle (Sec.~\ref{sec:dmd-opd}): the teacher--critic score difference updates
  the student. Bottom (Sec.~\ref{sec:critic-reweight}): a dynamics- and noise-conditioned predictor assigns
  mean-one weights to critic FM losses without changing their targets.
  Generator and critic share the timestep sampling rule; all auxiliary
  modules are training-only.}
  \label{fig:pipeline}
\end{figure*}

\subsection{Preliminaries: DMD and Its Challenges}
\label{sec:dmd-opd}

DMD~\cite{dmd,dmd2} distills a multi-step diffusion teacher into a
few-step generator $G_\theta$, which produces a latent rollout
$x_0^S=G_\theta(\xi,c)$ from noise $\xi$ and text conditions $c$.
It minimizes a time-averaged reverse KL between the re-noised student
and teacher distributions, yielding the score-difference update
\begin{equation}
  \nabla_\theta\mathcal L_G
  \approx\E\!\left[
  \left(\frac{\partial x_0^S}{\partial\theta}\right)^\top\!
  \bigl(s_\phi-s_T\bigr)\right],
  \label{eq:dmd-gradient}
\end{equation}
where $s_T$ is the frozen teacher score and $s_\phi$ is the student-score
estimate provided by an online fake-score model. Both are
evaluated at the re-noised rollout
$x_t=(1-\sigma_t)x_0^S+\sigma_t\varepsilon$, with noise level
$\sigma_t=\sigma(t)$ and independent standard Gaussian noises
$\xi,\varepsilon$.

Since the student score is not available in closed form, the critic
learns a velocity field $v_\phi$ that provides this estimate through
\begin{equation}
  s_\phi(x_t,t,c)
  =-\frac{x_t+(1-\sigma_t)v_\phi(x_t,t,c)}{\sigma_t}.
  \label{eq:score-velocity}
\end{equation}
With the generator fixed, $v_\phi$  is trained on detached student rollouts
using flow matching (FM)~\cite{lipman2022flow}:
\begin{equation}
  \mathcal L_{\mathrm{FM}}
  \propto\E\!\left[\left\|v_\phi(x_t,t,c)
  -(\varepsilon-x_0^S)\right\|_2^2\right].
  \label{eq:critic-fm}
\end{equation}
Through Eq.~\eqref{eq:score-velocity}, FM regression provides the
student-score estimate used in the generator update.
DMD alternates between generator and fake-score updates, using
independently sampled student rollouts for each update.

These alternating updates face two obstacles to preserving interaction dynamics.

\paragraph{Teacher-side mode locking:} for a re-noised student rollout
$x_t$, the teacher's clean-video posterior satisfies:
\begin{equation}
    p_T(x_0\mid x_t,c)\propto
    p_T(x_0\mid c)\mathcal N\!\left(
    x_t;(1-\sigma_t)x_0,\sigma_t^2\mathbf I\right).
\end{equation}
At high signal-to-noise ratio
$\operatorname{SNR}_t=(1-\sigma_t)^2/\sigma_t^2$,
the Gaussian likelihood favors clean videos near $x_0^S$,
so the posterior may remain localized around a motion-deficient
student rollout.
Since the teacher score is determined by this posterior's
mean~\cite{efron2011tweedie}, even an accurate score may provide
little guidance for recovering motion.
Appendices~\ref{app:dmd-details}--\ref{app:mode-locking} relate the score difference to posterior means and analyze this mode locking in a two-mode model.
Stronger re-noising relaxes this local constraint but weakens
appearance evidence, motivating re-noise sampling adapted to each
rollout's current interaction fidelity.
\paragraph{Critic lag:} the score estimate $s_\phi$ must track a changing
student distribution. With finitely many updates, $v_\phi$ may remain
inaccurate; Eq.~\eqref{eq:score-velocity} passes this error into the
generator's score difference~\cite{dmd2,rtdmd,fan2026phased}.
This motivates allocating limited critic fitting effort to rollouts that
are harder to track.

\subsection{Teacher-Velocity Turning Prior Schedule}
\label{sec:curvature}
Stronger re-noising can recover motion but may also introduce artifacts (Fig.~\ref{fig:teacher-recovery}). To guide timestep selection, we seek noise intervals where the teacher's denoising direction undergoes pronounced reorganization. Let $\widehat v_T(x_\sigma,\sigma,c)$ be its sampled velocity. The operator $C$
centers the velocity
field by subtracting the future-block temporal mean. Along trajectory $m$, the centered
velocity and its unit direction are
\begin{equation}
  v_m^{\mathrm c}(\sigma)=C\widehat v_T(x_{m,\sigma},\sigma,c_m),
  \qquad
  u_m(\sigma)=\frac{v_m^{\mathrm c}(\sigma)}
  {\|v_m^{\mathrm c}(\sigma)\|_2}.
  \label{eq:interaction-direction}
\end{equation}
We measure the change in teacher velocity direction per unit
noise level between adjacent sampling steps:
\begin{equation}
  \widehat\kappa_{m,k}
  =\frac{\arccos\!\left(u_m(\sigma_k)^\top u_m(\sigma_{k+1})\right)}
  {|\sigma_{k+1}-\sigma_k|}.
  \label{eq:interaction-turning}
\end{equation}
Its continuous limit is $\kappa_m(\sigma)=\|du_m/d\sigma\|_2$.
Small turning indicates a locally stable direction; large turning marks pronounced directional reorganization.

\paragraph{Posterior interpretation.}
To connect turning to clean-video predictions, use the conditional
field $v_T$ in Eq.~\eqref{eq:interaction-direction}. Along trajectory $m$,
the clean posterior and its mean satisfy
\begin{equation}
\begin{aligned}
  \pi_{m,\sigma}(x_0)&=p_T(x_0\mid x_{m,\sigma},c_m),\\
  \mu_{m,\sigma}&=x_{m,\sigma}-\sigma v_T(x_{m,\sigma},\sigma,c_m).
\end{aligned}
\label{eq:posterior-velocity}
\end{equation}
Along a flow trajectory, $dx/d\sigma=v_T$ implies
$d\mu/d\sigma=-\sigma\,dv_T/d\sigma$; thus velocity turning requires a
change in the posterior mean. For neighboring noise levels, the local KL
expansion~\cite{nielsen2020elementary} measures variation of the full posterior:
\begin{equation}
  \KL(\pi_{m,\sigma}\,\|\,\pi_{m,\sigma+\delta})
  =\tfrac12\mathcal I_m(\sigma)\delta^2+o(\delta^2),
  \label{eq:posterior-local-kl}
\end{equation}
where $\mathcal I_m(\sigma)$ is the posterior-path Fisher information, quantifying local posterior sensitivity to $\sigma$ along the trajectory.
The following result links this posterior variation to turning.

\begin{proposition}[Turning witnesses posterior variation]
\label{prop:curvature-fisher}
For regular posterior paths of the exact conditional flow, suppose
$\sigma\|v_m^{\mathrm c}(\sigma)\|_2$ is uniformly bounded away from zero.
Suppose the centered posterior covariance, projected orthogonally to $u_m$,
has uniformly bounded trace.
Then a constant $C_I>0$ exists such that
\begin{equation}
  \E_m[\mathcal I_m(\sigma)]\geq C_I\bar\kappa(\sigma)^2,
  \qquad \bar\kappa(\sigma)=\E_m[\kappa_m(\sigma)].
  \label{eq:mean-curvature-fisher}
\end{equation}
\end{proposition}

The bound motivates the squared mean turning rate as a geometric
indicator of posterior variation
(proof in Appendix~\ref{app:curvature-geometry}), supporting a sampling
prior that emphasizes noise intervals with pronounced denoising-direction
reorganization.
In practice, we profile 32 trajectories from the PF-Wan
teacher~\cite{physisforcing} using its official 40-step solver with a
CFG scale of 5. The measured turning rates are first averaged across
trajectories to estimate $\bar\kappa(\sigma)$, then smoothed over
$\sigma$ to obtain $\widetilde\kappa(\sigma)$
(Fig.~\ref{fig:curvature}).
The teacher-velocity-turning prior schedule is defined on $\mathcal T$,
the timestep range used for re-noising:
\begin{equation}
  h_\kappa(t)=(\widetilde\kappa(\sigma_t)+\epsilon_\kappa)^2,
  \qquad
  \nu_\kappa(t)=\frac{h_\kappa(t)}
  {\int_{\mathcal T}h_\kappa(u)\,du}.
  \label{eq:curvature-query-measure}
\end{equation}
The offset $\epsilon_\kappa>0$ preserves support.

\begin{figure}[t]
  \centering
  \includegraphics[width=0.85\columnwidth]{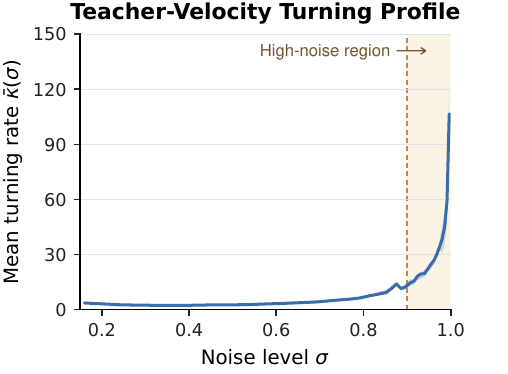}
  \caption{\textbf{Teacher-velocity turning profile.} Mean turning rate of
  future-block-centered velocity across PF-Wan teacher trajectories. The light-blue band shows the interquartile range across trajectories; pale-yellow shading highlights the high-noise region $\sigma\in[0.9,1]$.}
  \label{fig:curvature}
  \vspace{-2mm}
\end{figure}

\subsection{\texorpdfstring{\raggedright Temporal Affinity--Conditioned \mbox{Re-noise Sampling}}{Temporal Affinity--Conditioned Re-noise Sampling}}
\label{sec:router}

The teacher prior $\nu_\kappa$ emphasizes noise intervals with pronounced
denoising-direction reorganization, but is fixed across rollouts.
We adapt its mixture with the base schedule $\nu_{\mathrm{base}}$ to each rollout's current
interaction fidelity. To assess this fidelity, we compare multi-scale temporal changes between the decoded student video $V_i^S$ and its paired target $V_i^\star$ in a frozen video encoder's feature space. Their pooled features and temporal differences are
\begin{equation}
\begin{aligned}
  Z_i^a&=\Pool(\Phi(V_i^a)),\qquad a\in\{S,\star\},\\
  (\Delta_\ell Z_i^a)_{j,u}
  &=(Z_i^a)_{j+\ell,u}-(Z_i^a)_{j,u},
\end{aligned}
\label{eq:routing-features}
\end{equation}
where $\Phi$ is the frozen encoder, instantiated as V-JEPA~2.1~\cite{vjepa21}, and $\Pool$ normalizes and pools tokens within a $3\times3$ spatial grid. The indices $j,u$ denote time and spatial cells, and $\ell\in\{1,2,4\}$ is the temporal lag.
For brevity, write $\delta z_i^a=(\Delta_\ell Z_i^a)_{j,u}$ and let $\langle\cdot\rangle$ denote lag-weighted averaging over time--cell pairs. We use their \emph{temporal affinity} $A_i$ as a feature-space proxy for interaction fidelity:
\begin{equation}
  A_i=
  \frac{2\left\langle[(\delta z_i^S)^\top\delta z_i^\star]_+\right\rangle}
  {\left\langle\|\delta z_i^S\|_2^2+\|\delta z_i^\star\|_2^2\right\rangle},
  \label{eq:motion-score}
\end{equation}
where $[\,\cdot\,]_+=\max(\cdot,0)$ and $A_i\in[0,1]$. Identical temporal changes give one, while mismatches in direction or
magnitude reduce affinity.

The affinity sets the teacher-prior probability $g_i$ and timestep sampling density:
\begin{equation}
\begin{aligned}
  g_i&=\frac{\alpha}{\alpha+(1-\alpha)A_i},\qquad 0<\alpha\leq1,\\
  \nu_i(t)&=(1-g_i)\nu_{\mathrm{base}}(t)+g_i\nu_\kappa(t).
\end{aligned}
\label{eq:routed-density}
\end{equation}
Here $\alpha$ sets the minimum teacher-prior probability. Low-affinity rollouts receive more teacher-prior
sampling, while well-matched rollouts retain more base-schedule coverage
for appearance refinement. Both generator and critic sample $t_i\sim\nu_i$, with $\nu_i$ computed from their respective rollouts. This aligns critic training with the noise-level allocation used for generator updates.

\begin{table*}[t]
  \centering
  \caption{\textbf{Main comparison on embodied-video benchmarks.}
  The teacher and undistilled student are reference models and are separated from the two
  matched four-step DMD variants. Bold with blue shading marks the better result within the
  matched pair. NFE counts denoiser evaluations, including both CFG branches. Benchmark
  scores are percentages and higher is better. R-Bench embodiment columns are abbreviated
  as Dual (dual-arm), Hum. (humanoid), Single (single-arm), and Quad. (quadruped).}
  \label{tab:main_comparison}

  \small
  \setlength{\tabcolsep}{2pt}
  \renewcommand{\arraystretch}{1.08}
  \begin{tabular*}{\textwidth}{@{\extracolsep{\fill}}l*{16}{c}@{}}
    \toprule
    \multirow{2}{*}{Method} & \multirow{2}{*}{Params} & \multirow{2}{*}{NFE}
    & \multicolumn{5}{c}{R-Bench: TAC}
    & \multicolumn{5}{c}{PAI-Bench-G}
    & \multicolumn{4}{c}{EZS-Bench} \\
    \cmidrule(lr){4-8}\cmidrule(lr){9-13}\cmidrule(l){14-17}
    & & & Overall & Dual & Hum. & Single & Quad.
    & Domain & Space & Physics & Time & Quality
    & Domain & Space & Physics & Time \\
    \midrule
    PF-Wan & 14B & 80
      & 75.3 & 72.5 & 79.8 & 71.8 & 77.2
      & 89.0 & 91.8 & 92.2 & 88.4 & 76.9
      & 83.2 & 80.9 & 89.2 & 78.0 \\
    Wan2.1-Fun & 1.3B & 100
      & 47.3 & 42.5 & 53.2 & 30.8 & 63.2
      & 79.0 & 79.7 & 82.5 & 76.4 & 75.3
      & 75.5 & 71.7 & 83.9 & 68.8 \\
    \midrule
    Base DMD & 1.3B & 4
      & 34.6 & 25.2 & 33.4 & 25.8 & 54.0
      & 74.4 & 77.0 & \cellcolor{intdrei!70}\textbf{82.8} & 64.7
      & 78.1
      & 73.3 & 66.9 & \cellcolor{intdrei!70}\textbf{86.0} & 64.8 \\
    \textbf{\method} & 1.3B & 4
      & \cellcolor{intdrei!70}\textbf{44.2}
      & \cellcolor{intdrei!70}\textbf{38.8}
      & \cellcolor{intdrei!70}\textbf{48.0}
      & \cellcolor{intdrei!70}\textbf{31.2} 
      & \cellcolor{intdrei!70}\textbf{58.8}
      & \cellcolor{intdrei!70}\textbf{79.5}
      & \cellcolor{intdrei!70}\textbf{82.4} 
      & 81.2
      & \cellcolor{intdrei!70}\textbf{73.5}
      & \cellcolor{intdrei!70}\textbf{78.1}
      & \cellcolor{intdrei!70}\textbf{76.2}
      & \cellcolor{intdrei!70}\textbf{71.8} & 84.4
      & \cellcolor{intdrei!70}\textbf{70.8} \\
    \bottomrule
  \end{tabular*}
\end{table*}

\begin{figure*}[t!]
  \centering
  \includegraphics[width=.98\textwidth]{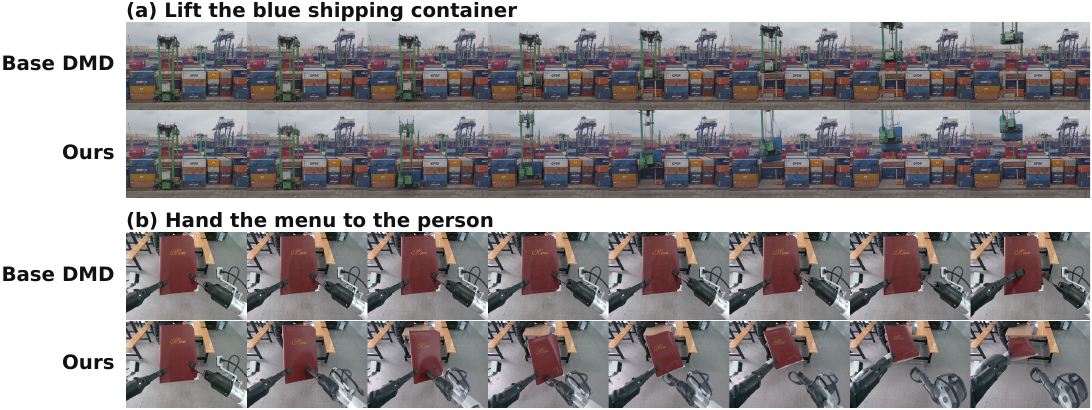}
  \caption{\textbf{Visual comparison of interaction dynamics between Base DMD and \method at four NFE.} Uniformly spaced frames from two
  PAI-Bench cases. Top: lifting a blue shipping container. Bottom:
  handing a menu to a person. Base DMD fails to lift the target container and produces
  limited handover motion, whereas \method realizes the instructed interaction.}
  \label{fig:qualitative}
  \vspace{-1mm}
\end{figure*}

\subsection{Dynamics-Guided Fake-Score Tracking}
\label{sec:critic-reweight}

Even at a fixed noise level, critic fitting difficulty varies across rollouts: stronger latent motion is associated with higher FM loss (Fig.~\ref{fig:motion-tracking-probe}). We  therefore use latent temporal dynamics to predict
noise-relative fitting difficulty and allocate
fitting effort within a fixed critic-update budget. 

To compare rollouts evaluated at different noise levels, we center
log-FM loss by a noise-bin baseline $b_{k(t_i)}$, an exponential
moving average of log-FM loss in the bin containing $t_i$. We define
the \emph{noise-relative fitting difficulty} as
\begin{equation}
  r_i=\sg[\log\ell_i^{\mathrm{FM}}-b_{k(t_i)}].
\label{eq:difficulty-prediction}
\end{equation}
A positive $r_i$ indicates higher log-FM loss than the running
baseline at a comparable noise level.

We predict this difficulty from VAE latent dynamics, measured in the
critic's FM regression space. For a detached student latent
$X_i=\sg[x_{0,i}^S]$, a fixed map $\mathcal D_{\mathrm{temp}}$ aggregates
RMS magnitudes of multi-lag latent differences and displacements from
the first block using temporal statistics such as the mean and
standard deviation. A lightweight MLP $f$ maps the descriptor $d_i$
and a noise-level embedding $e(t_i)$ to the predicted fitting difficulty
$\widehat r_i$:
\begin{equation}
\begin{aligned}
  d_i&=\mathcal D_{\mathrm{temp}}(X_i),\\
  \widehat r_i&=f(d_i,e(t_i)).
\end{aligned}
\label{eq:critic-descriptor}
\end{equation}
We train $f$ online with a Huber loss $\mathcal L_{\mathrm{pred}}$
against $r_i$. Descriptor details are given in
Appendix~\ref{app:critic-details}.

To emphasize predicted-hard rollouts within the existing critic updates,
we exponentiate $\widehat r_i$ and normalize the weights to unit mean
over the global minibatch of size $B$:
\begin{equation}
  w_i=\frac{\exp(\beta\widehat r_i)}
  {B^{-1}\sum_{j=1}^B\exp(\beta\widehat r_j)}.
\label{eq:critic-sample-weight}
\end{equation}
Here $\beta>0$ controls the strength on predicted-hard rollouts.
The training objective is
\begin{equation}
  \mathcal L_C=\frac1B\sum_{i=1}^B\sg[w_i]\,\ell_i^{\mathrm{FM}}(\phi)
  +\lambda_{\mathrm{pred}}\mathcal L_{\mathrm{pred}}.
  \label{eq:difficulty-critic-objective}
\end{equation}
Rollouts with higher predicted fitting difficulty receive larger coefficients in the
critic gradient. Detached weights ensure that $f$ learns only
from $\mathcal L_{\mathrm{pred}}$ and all auxiliary modules are training-only.

\section{Experiments}
\label{sec:experiments}

\subsection{Experimental Setup}
\label{sec:setup-exp}

\paragraph{Models and training.}

We distill the PF-Wan 14B image-to-video teacher~\cite{physisforcing} into a bidirectional 1.3B student. PF-Wan is fine-tuned on robotic manipulation videos with interaction-focused physics supervision, making it a suitable teacher for transferring interaction dynamics to a few-step student.
The student and fake-score model are initialized from Wan2.1-Fun-V1.1-1.3B-InP~\cite{aigc_apps_VideoX_Fun_2026}.
The teacher uses its official 40-step, shift-5 sampler with CFG 5;
the student generates 81-frame videos in four denoiser evaluations.
A frozen V-JEPA 2.1-L encoder provides temporal-affinity features.
We train on 32,897 clips curated from RoVid-X~\cite{rovidx}
through semantic deduplication, image--text alignment, and
motion/visibility filtering.
All DMD variants use the same 2,500-step training budget and
effective batch size of 32. Tracking reweights samples without
changing the number of critic updates.
We set the teacher-prior probability
floor to $\alpha=0.3$ and critic-weighting strength to $\beta=0.5$.
Data curation and implementation details are provided in
Appendices~\ref{sec:data} and~\ref{sec:impl}, with
hyperparameters listed in Table~\ref{tab:method-hparams}.

\begin{table}[t]
  \centering
  \caption{\textbf{Sampling and fake-score tracking ablation.}
  Sampling denotes temporal affinity--conditioned re-noise sampling; Tracking denotes
  dynamics-guided critic weighting. All variants retain the critic and use
  the same update and inference budgets. Bold with blue shading marks
  the best result in each column.}
  \label{tab:component_ablation}

  \scriptsize
  \setlength{\tabcolsep}{2pt}
  \renewcommand{\arraystretch}{1.12}
  \begin{tabular*}{\columnwidth}{@{\extracolsep{\fill}}lcccccc@{}}
    \toprule
    \multirow{2}{*}{Method}
    & \multicolumn{2}{c}{Components}
    & \multicolumn{1}{c}{R-Bench}
    & \multicolumn{2}{c}{PAI-Bench-G}
    & \multicolumn{1}{c}{EZS-Bench} \\
    \cmidrule(lr){2-3}\cmidrule(lr){4-4}\cmidrule(lr){5-6}\cmidrule(l){7-7}
    & Sampling & Tracking & TAC & Domain & Quality & Domain \\
    \midrule
    Base DMD
    & $\times$ & $\times$ & 34.6
    & 74.4 & 78.1 & 73.3 \\
    w/o Tracking
    & $\checkmark$ & $\times$ & 43.5
    & 78.5 & 78.1 & 75.1 \\
    w/o Sampling
    & $\times$ & $\checkmark$ & 38.3
    & 75.4 & 78.0 & 74.0 \\
    \textbf{\method}
    & $\checkmark$ & $\checkmark$ 
    & \cellcolor{intdrei!70}\textbf{44.2}
    & \cellcolor{intdrei!70}\textbf{79.5} 
    & \cellcolor{intdrei!70}\textbf{78.1} 
    & \cellcolor{intdrei!70}\textbf{76.2} \\
    \bottomrule
  \end{tabular*}
\end{table}

\paragraph{Evaluation.}
For in-domain evaluation, we report Task-Adherence Consistency (TAC)
on 400 held-out image--prompt pairs from R-Bench~\cite{rovidx},
spanning four robot embodiments.
We assess cross-dataset generalization on the robot subset of
PAI-Bench-G~\cite{zhou2026pai} (174 samples) and
EZS-Bench~\cite{abotphysworld} (196 samples), reporting the Domain
score alongside Space, Physics, and Time scores.
We also report the official PAI-Bench-G Quality score to assess
visual quality.
All models are evaluated with the same initial images, prompts,
and random seeds.
For EZS-Bench, we follow PhysisForcing~\cite{physisforcing}
in compressing the original descriptions into action-centric prompts
for video generation.
Detailed evaluation protocols are provided in
Appendix~\ref{sec:benchmark-context}.

\subsection{Main Results}
\label{sec:main}

As shown in Table~\ref{tab:main_comparison}, our student uses 4 NFE, versus 80 for the teacher and 100
for the undistilled student.
Base DMD, which uses the base re-noise schedule and uniform critic
weights, substantially degrades interaction fidelity: R-Bench TAC falls to
$34.6$ and PAI Domain to $74.4$, despite an increase in PAI Quality
from $75.3$ to $78.1$.
These results indicate that Base DMD preserves visual quality but does not adequately retain task-relevant interaction dynamics.

At the same four-NFE budget, \method better preserves interaction dynamics, raising R-Bench TAC to $44.2$ ($+9.6$ points), and PAI/EZS Domain to
$79.5/76.2$ ($+5.1/+2.9$ points). Base DMD scores $1.6$ points higher on Physics in PAI-Bench-G because its near-static outputs preserve plausible physical states without executing the instructed interaction (Fig.~\ref{fig:qualitative}).

\begin{figure}[t]
  \centering
  \includegraphics[width=\columnwidth]{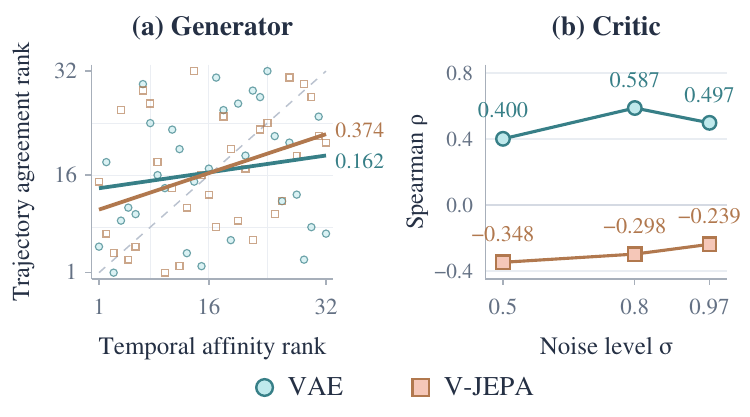}
  \caption{\textbf{Representations for generator sampling and critic tracking.}
  (a) Rank comparison between temporal affinity and CoTracker3-based trajectory agreement across 32 student--target video pairs. Solid lines show linear fits and the dashed diagonal denotes identical rankings. (b) Spearman $\rho$ between MLP predictions and log-FM loss, computed separately at each noise level on a 32-rollout bank. FM losses are averaged over three independent noise draws per rollout and noise level.}
  \label{fig:representation-probes}
\end{figure}

\subsection{Representation Analysis for the Generator and Critic}
\label{sec:representation-analysis}

We compare VAE and V-JEPA representations for two training signals:
target-relative temporal affinity for generator sampling and noise-relative
fitting difficulty for critic tracking. Each probe uses a frozen
32-rollout bank.

\paragraph{Generator representations.}
We compute temporal affinity $A_i$ (Section~\ref{sec:router}) in both spaces.
As an independent motion reference, we compute a
\emph{trajectory agreement} score $T_i$ using point trajectories
tracked by CoTracker3~\cite{cotracker3}.
This score measures how closely the generated video matches
the target video in the directions and magnitudes of point displacements.
Figure~\ref{fig:representation-probes}(a) compares the rankings
induced by $A_i$ and $T_i$.
V-JEPA achieves a higher Spearman correlation with trajectory
agreement than VAE ($0.374$ vs.\ $0.162$), indicating that its
temporal affinity better reflects target-relative motion
consistency on these rollouts.

\paragraph{Critic representations.}
We fit small offline MLPs to predict log-FM loss from temporal
descriptors and evaluate them on held-out rollouts.
Figure~\ref{fig:representation-probes}(b) reports Spearman
correlations computed separately at each noise level
$\sigma\in\{0.5,0.8,0.97\}$.
VAE-based predictions are positively correlated with the
measured loss across all tested noise levels, whereas
V-JEPA-based predictions are consistently negatively correlated.
At a fixed noise level, subtracting the shared baseline
preserves the ranking of log-FM losses, so this probe evaluates
the sample ordering for noise-relative fitting difficulty.
These results support VAE descriptors for critic tracking,
complementing V-JEPA features for target-relative motion assessment.
Appendix~\ref{app:head-ranking} evaluates the trained
online predictor. Its predictions remain strongly correlated with measured FM loss, supporting its ability
to distinguish samples by critic fitting difficulty.

\subsection{Ablation Study}
\label{sec:ablation}

\paragraph{Re-noise sampling and fake-score tracking.}
Table~\ref{tab:component_ablation} evaluates the individual and combined
effects of two components.
Sampling alone raises R-Bench TAC/PAI Domain from $34.6/74.4$ to
$43.5/78.5$, while tracking alone achieves $38.3/75.4$.
Adding tracking to adaptive sampling further improves R-Bench TAC/PAI Domain
to $44.2/79.5$ and EZS Domain from $75.1$ to $76.2$, with PAI Quality
unchanged at $78.1$.
These results show that fake-score tracking complement adaptive re-noise sampling.

\begin{table}[t]
  \centering
  \caption{\textbf{Re-noise sampling-schedule ablation.}
  We compare Base DMD with a high-noise control (Shift-25, raising the flow
  shift from 5 to 25), the teacher prior, and our temporal
  affinity--conditioned sampler.
  All variants disable critic reweighting ($\beta=0$).
  Bold with blue shading marks the best result in each column.}
  \label{tab:routing_ablation}
  \footnotesize
  \setlength{\tabcolsep}{3.5pt}
  \renewcommand{\arraystretch}{1.10}
  \begin{tabular*}{\columnwidth}{@{\extracolsep{\fill}}lcccc@{}}
    \toprule
    \multirow{2}{*}{Sampling schedule}
    & \multicolumn{1}{c}{R-Bench}
    & \multicolumn{2}{c}{PAI-Bench-G}
    & \multicolumn{1}{c}{EZS-Bench} \\
    \cmidrule(lr){2-2}\cmidrule(lr){3-4}\cmidrule(l){5-5}
    & TAC & Domain & Quality & Domain \\
    \midrule
    Base DMD        & 34.6 & 74.4 & 78.1 & 73.3 \\
    Shift-25        & 40.0 & 77.2 & \cellcolor{intdrei!70}\textbf{78.2} & 74.6 \\
    Teacher prior   & 42.5  & 77.9 & 78.0 & 75.0 \\
    Affinity-conditioned & \cellcolor{intdrei!70}\textbf{43.5}
                    & \cellcolor{intdrei!70}\textbf{78.5} 
                    & 78.1
                    & \cellcolor{intdrei!70}\textbf{75.1} \\
    \bottomrule
  \end{tabular*}
\end{table}

\begin{table}[t]
  \centering
  \caption{\textbf{Sensitivity to critic-weighting strength $\beta$.}
  We vary only $\beta$, with $\alpha=0.3$. The $\beta=0.5$ row repeats
  the full-model results from Table~\ref{tab:main_comparison}.
  Bold with blue shading marks column-wise best values, including ties.}
  \label{tab:beta_sensitivity}
  \footnotesize
  \setlength{\tabcolsep}{3.5pt}
  \renewcommand{\arraystretch}{1.10}
  \begin{tabular*}{\columnwidth}{@{\extracolsep{\fill}}lcccc@{}}
    \toprule
    \multirow{2}{*}{$\beta$}
    & \multicolumn{1}{c}{R-Bench}
    & \multicolumn{2}{c}{PAI-Bench-G}
    & \multicolumn{1}{c}{EZS-Bench} \\
    \cmidrule(lr){2-2}\cmidrule(lr){3-4}\cmidrule(l){5-5}
    & TAC & Domain & Quality & Domain \\
    \midrule
    $0.1$ & 43.8 & 79.1 & 78.0& 75.4\\
    $0.5$ & \cellcolor{intdrei!70}\textbf{44.2}
          & \cellcolor{intdrei!70}\textbf{79.5}
          & \cellcolor{intdrei!70}\textbf{78.1}
          & \cellcolor{intdrei!70}\textbf{76.2} \\
    $1.0$ & 43.6 & 78.4 & 78.1 & 75.1 \\
    \bottomrule
  \end{tabular*}
\end{table}

\paragraph{Re-noise sampling schedules.}

Table~\ref{tab:routing_ablation} compares re-noise schedules with critic reweighting
disabled ($\beta=0$). We choose Shift-25 to approximately match the teacher prior's
probability mass at $\mathrm{SNR}<1$ (Appendix~\ref{sec:curvature-sampling}).
Increasing the shift from $5$ to $25$ raises R-Bench TAC/PAI Domain from
$34.6/74.4$ to $40.0/77.2$, supporting stronger re-noising for motion recovery.
The teacher prior further improves them to $42.5/77.9$ at similar total
high-noise mass, supporting timestep sampling informed by the teacher's flow
geometry across noise levels. Affinity conditioning achieves the best
R-Bench TAC and Domain scores ($43.5/78.5/75.1$), supporting rollout-specific
adaptation beyond a fixed prior.

\paragraph{Critic-weighting strength.}

At fixed $\alpha=0.3$ (Table~\ref{tab:beta_sensitivity}), both
$\beta=0.1$ and $0.5$ improve TAC and both Domain scores over
uniform critic weighting, supporting moderate emphasis on
motion-related hard rollouts. Increasing $\beta$ from $0.5$ to
$1.0$ reduces R-Bench TAC and Domain scores, suggesting that excessive emphasis on hard
rollouts compromise score estimation on typical samples
under a fixed update budget.

\subsection{Qualitative Results}
\label{sec:qualitative}
As illustrated in Fig.~\ref{fig:qualitative}, base DMD often produces nearly static videos or limited gripper motion
without completing the intended interaction. In contrast, \method
generates coordinated robot--object motion while preserving scene
appearance. In the container-lifting example, Base DMD moves the lifting
mechanism without lifting the target container, whereas \method lifts
the blue container from the stack. In the menu-handover example,
Base DMD moves the grippers but fails to complete the handover,
whereas \method depicts the robot successfully passing the menu
to the person.
\begin{table}[t]
  \centering
  \caption{\textbf{WorldArena action-planner success rate (\%).}
  Upper block: results reported in PhysisForcing~\cite{physisforcing}. Lower block: our 1.3B backbones with separately trained,
  task-specific VPP heads. Bold with blue shading compares the lower block only.}
  \label{tab:downstream}
  \small
  \setlength{\tabcolsep}{4pt}
  \begin{tabular*}{\columnwidth}{@{\extracolsep{\fill}}lrrr@{}}
    \toprule
    Model & \shortstack{Adjust\\Bottle} & \shortstack{Click\\Bell} & Avg. \\
    \midrule
    Genie Envisioner & 10.0 & 20.0 & 15.0 \\
    TesserAct & 1.0 & 35.0 & 18.0 \\
    RoboMaster & 8.0 & 20.0 & 14.0 \\
    Vidar & 2.0 & 19.0 & 10.5 \\
    WoW & 20.0 & 21.0 & 20.5 \\
    Wan2.2-TI2V-5B & 12.0 & 20.0 & 16.0 \\
    \midrule
    Base DMD (Wan2.1-1.3B) & 1.0 & 31.0 & 16.0 \\
    \textbf{\method} (Wan2.1-1.3B) & \cellcolor{intdrei!70}\textbf{2.0}
      & \cellcolor{intdrei!70}\textbf{66.0}
      & \cellcolor{intdrei!70}\textbf{34.0} \\
    \bottomrule
  \end{tabular*}
\end{table}

\subsection{Downstream Action Planning}
\label{sec:downstream}

Following WorldArena action-planner protocol~\cite{worldarena}, we compare Base DMD and \method as frozen 1.3B backbones on Adjust Bottle and
Click Bell in RoboTwin~2.0~\cite{robotwin2}. For each frozen 1.3B backbone, we train task-specific
VPP heads~\cite{vpp} to predict actions from intermediate DiT features, using 50 demonstrations per task for 200 epochs. We report closed-loop success rates over 100 episodes per task. Implementation details are in Appendix~\ref{app:downstream}.

\begin{samepage}
Using \method as the backbone raises mean success from Base DMD's
$16.0\%$ to $34.0\%$ (Table~\ref{tab:downstream}).
The gain comes mainly from Click Bell ($31.0\%\to66.0\%$),
with a smaller increase on Adjust Bottle ($1.0\%\to2.0\%$).
\end{samepage}

\section{Conclusion}
\label{sec:conclusion}

We presented \method for preserving interaction dynamics in few-step video world models. Temporal affinity--conditioned re-noise sampling
adapts the mixture of the teacher-velocity-turning prior and the base schedule to each rollout's current interaction fidelity.
To improve score estimation for rollouts with stronger motion,
dynamics-guided fake-score tracking reweights the critic loss
according to predicted noise-relative fitting difficult. The resulting four-step 1.3B student improves R-Bench TAC and
Domain scores on PAI-Bench-G and EZS-Bench over Base DMD without auxiliary inference modules. As a frozen backbone for downstream action planning, the student also improves the mean success rate from Base DMD's $16\%$ to $34\%$ across two WorldArena tasks.
\paragraph{Limitations.}
\label{sec:limitations}
The turning profile is teacher- and schedule-specific and must be re-estimated for a new teacher.
\clearpage
{
    \footnotesize
    \bibliographystyle{ieeenat_fullname}
    \bibliography{main}
}

\clearpage
\setcounter{page}{1}
\renewcommand{\thepage}{S\arabic{page}}
\maketitlesupplementary
\appendix
\setcounter{table}{0}
\renewcommand{\thetable}{S\arabic{table}}
\setcounter{figure}{0}
\renewcommand{\thefigure}{S\arabic{figure}}

\section{Training-Set Curation}
\label{sec:data}

We curate a training set of 32,897 clips from the 2,824,039-clip
RoVid-X corpus~\cite{rovidx} through three stages, prioritizing
semantic diversity and measurable interaction motion.

\paragraph{Stage 1: semantic caption deduplication.}
We first remove exact duplicates from the \texttt{short\_caption}
field, obtaining 1,374,834 distinct captions. We then encode these
captions with a ViT-L/14 text encoder and $\ell_2$-normalize the
embeddings. To remove semantically redundant captions, we apply
a greedy covering procedure with a cosine similarity threshold
of 0.95. We process captions in descending order of frequency,
retaining each remaining caption as a representative and
discarding all remaining captions whose similarity to it is
at least 0.95. Unlike connected-component clustering, this
procedure avoids merging captions through chains of intermediate
matches and retains 448,528 representatives.

\paragraph{Stage 2: image--text filtering.}
For each retained clip, we decode the first frame and compute
its similarity to the caption using ViT-L/14 image and text
embeddings. We retain the top 50\% by similarity, leaving
224,264 clips. This stage filters visual--textual mismatches
that caption-only deduplication cannot detect.

\paragraph{Stage 3: motion and visibility filtering.}
We track points with CoTracker3~\cite{cotracker3} and retain
clips with at least 50 selected tracks and a mean visibility
of at least 0.6. This retains 14.7\% of the remaining clips,
yielding the final training set of 32,897 clips.

\begin{table}[t]
  \centering
  \small
  \setlength{\tabcolsep}{4pt}
  \begin{tabular}{llrr}
    \toprule
    Stage & Criterion & Input & Output \\
    \midrule
    Deduplication & greedy cover, $0.95$ & 2,824,039 & 448,528 \\
    Image--text & top-$50\%$ cosine & 448,528 & 224,264 \\
    Motion & tracks$\geq50$, vis.$\geq0.6$ & 224,264 & 32,897 \\
    \bottomrule
  \end{tabular}
  \caption{Training-data curation funnel.}
  \label{tab:funnel}
\end{table}

\section{Implementation Details}
\label{sec:impl}

\subsection{Image-to-video conditioning}

The 1.3B student and critic use the Wan2.1-Fun I2V
interface~\cite{aigc_apps_VideoX_Fun_2026} and share the
16-channel Wan2.1 latent space with the PF-Wan teacher.
For image conditioning, we construct a video containing
the input image as its first frame and zeros in all subsequent
frames, then encode it with the Wan2.1 VAE. An accompanying
mask marks only the first latent temporal block as observed.

The student performs four bidirectional denoising evaluations at model times
$\{1,0.75,0.5,0.25\}$ with flow shift 5. It generates 21 latent temporal blocks, decoded to
81 frames at 16 FPS. The teacher uses its official 40-step, shift-5, CFG-5 schedule.
During DMD, teacher-score evaluations use CFG 5; the critic is conditional without CFG.

\subsection{Video features for temporal affinity}
\label{app:affinity-features}

We sample 16 RGB frames uniformly over the 81-frame video. V-JEPA produces
an $8\times24\times32\times1024$ token grid. We normalize its tokens, pool them
within a fixed $3\times3$ full-frame grid, and normalize the pooled features.Paired targets are cached as 8 × 9 × 1024 fp16 features.

\subsection{Teacher prior schedule and re-noise sampling}
\label{sec:curvature-sampling}

The base DMD schedule samples raw model time uniformly
and applies the flow-shift transformation with shift 5.
For $u\sim\mathcal U[0,1)$, the shift and physical noise coordinate are
\begin{equation}
  \sigma_t=t=f_s(u)=1000\,\frac{s\,(u)}{1+(s-1)(u)}
\end{equation}
We sample re-noising timesteps within $\mathcal T=[0.020,0.999)$. For $s>1$, the shift moves raw
times toward higher noise. 

For the Shift-25 control in Table~\ref{tab:routing_ablation}, we choose $s=25$ to approximately match the teacher prior's
mass in the high-noise region $\mathrm{SNR}_t<1$. On the common interval $\mathcal T=[0.020,0.999)$, the normalized
CDFs give probabilities of $96.14\%$ for Shift-25 and $95.69\%$ the
teacher prior.

\section{Geometry and Posterior Interpretation of Teacher-Velocity Turning}
\label{app:curvature-geometry}

We relate DMD's score difference to posterior means, analyze high-SNR
mode locking in a two-mode model, and prove the connection between
teacher-velocity turning and posterior variation.

\subsection{Score difference and posterior means}
\label{app:dmd-details}

Let $q_{\theta,t}$ and $p_{T,t}$ denote the noisy marginals obtained by
applying the re-noising process in Section~\ref{sec:dmd-opd} to the
clean student and teacher distributions, with exact scores $s_{q,t}$
and $s_T$.
Let $\mu_{q,t}$ and $\mu_{T,t}$ be the corresponding clean posterior
means conditioned on the same $(x_t,c)$. Gaussian conditioning gives
\begin{equation}
  s_{q,t}-s_T
  =\frac{1-\sigma_t}{\sigma_t^2}(\mu_{q,t}-\mu_{T,t}).
  \label{eq:posterior-score}
\end{equation}
Thus, the exact score difference depends on the two posterior means;
the analysis below shows how high SNR can cause the teacher posterior to
concentrate near a motion-deficient student rollout.

\subsection{Mode-Locking Analysis}
\label{app:mode-locking}

\begin{proposition}[High-SNR mode locking]
\label{prop:mode-locking-appendix}
Consider a teacher with two clean modes
$p_{T,0}=\pi_s\delta_{x_s}+\pi_m\delta_{x_m}$, where $\delta_x$ is a Dirac measure,
$\pi_s,\pi_m>0$, $\pi_s+\pi_m=1$, $x_s$ is a static mode, and $x_m$ is a moving mode.
Re-noise a collapsed student sample as
$x_t=(1-\sigma_t)x_s+\sigma_t\varepsilon$, $\varepsilon\sim\mathcal N(0,\mathbf I)$, and let
$d=x_m-x_s$ and $\operatorname{SNR}_t=(1-\sigma_t)^2/\sigma_t^2$. The teacher posterior odds
and their expected log value satisfy
\begin{equation}
  \begin{aligned}
  \log\frac{p_T(x_m\mid x_t,c)}{p_T(x_s\mid x_t,c)}
  &=\log\frac{\pi_m}{\pi_s}
   -\frac{\operatorname{SNR}_t}{2}\|d\|_2^2\\
   &\quad+\frac{1-\sigma_t}{\sigma_t}\langle\varepsilon,d\rangle,\\
  \E_\varepsilon\!\left[\log\frac{p_T(x_m\mid x_t,c)}
  {p_T(x_s\mid x_t,c)}\right]
  &=\log\frac{\pi_m}{\pi_s}
   -\frac{\operatorname{SNR}_t}{2}\|d\|_2^2.
  \end{aligned}
  \label{eq:app-posterior-odds}
\end{equation}
The expected log-odds of the moving mode decrease with SNR and squared
mode separation, favoring the static mode at high SNR.
\end{proposition}
\begin{proof}
Bayes' rule reduces the posterior odds to the prior odds times the ratio of the two Gaussian
likelihoods. Expanding their squared distances gives Eq.~\eqref{eq:app-posterior-odds}.
\end{proof}

\subsection{Posterior Variation and Proof of
Proposition~\ref{prop:curvature-fisher}}

\paragraph{Setup.}
For $L$ latent temporal blocks $v=(v_1,\ldots,v_L)$,
define the linear operator
\begin{equation}
  Hv=\operatorname{vec}\Big[
  v_j-\tfrac{1}{L-1}\textstyle\sum_{\ell=2}^L v_\ell
  \Big]_{j=2}^L.
  \label{eq:temporal-projection}
\end{equation}
The operator $H$ excludes the observed block and subtracts
the temporal mean over the future blocks.

We consider the linear noising process
$x_\sigma=(1-\sigma)x_0+\sigma\varepsilon$, where
$\varepsilon$ is standard Gaussian noise independent of
$x_0$. Let $v_T$ denote its exact conditional velocity
field. Along a characteristic $x_{m,\sigma}$ satisfying
$d x_{m,\sigma}/d\sigma
=v_T(x_{m,\sigma},\sigma,c_m)$, define
\begin{equation}
\begin{aligned}
  \pi_{m,\sigma}(x_0)
  &=p_T(x_0\mid x_{m,\sigma},c_m),\\
  \mathcal I_m(\sigma)
  &=\E_{\pi_{m,\sigma}}\!\left[
  \left(
  \frac{d}{d\sigma}\log\pi_{m,\sigma}(x_0)
  \right)^2
  \right].
\end{aligned}
\label{eq:posterior-path-fisher}
\end{equation}
Here the derivative follows the trajectory
$x_{m,\sigma}$. Write
$\mu_{m,\sigma}=\E_{\pi_{m,\sigma}}[x_0]$,
$\Sigma_{m,\sigma}
=\operatorname{Cov}_{\pi_{m,\sigma}}(x_0)$, and
$v_{m,\sigma}=v_T(x_{m,\sigma},\sigma,c_m)$.
Define the centered velocity and its direction as
\[
  v_m^{\mathrm c}(\sigma)=Hv_{m,\sigma},
  \qquad
  u_m(\sigma)=
  \frac{v_m^{\mathrm c}(\sigma)}
       {\|v_m^{\mathrm c}(\sigma)\|_2},
\]
with $P_{m,\sigma}=\mathbf I-u_mu_m^\top$ and
$\kappa_m(\sigma)=\|\dot u_m(\sigma)\|_2$.
Dots denote total derivatives along the characteristic.
We assume a regular posterior path, $\sigma>0$,
nonzero centered velocity, and a positive projected
covariance trace.

\paragraph{Pointwise bound.}
For the linear noising process, the exact conditional
velocity satisfies
\begin{equation}
  \begin{aligned}
  v_{m,\sigma}
  &=\E_T[\varepsilon-x_0\mid x_{m,\sigma},c_m]\\
  &=\frac{x_{m,\sigma}-\mu_{m,\sigma}}{\sigma}.
  \end{aligned}
  \label{eq:app-conditional-flow}
\end{equation}
Differentiating
$\mu_{m,\sigma}=x_{m,\sigma}-\sigma v_{m,\sigma}$
along the characteristic gives
\begin{equation}
  \dot\mu_{m,\sigma}
  =-\sigma\dot v_{m,\sigma}.
  \label{eq:app-mean-velocity}
\end{equation}
Since $H$ is fixed, differentiation of the normalized
centered velocity yields
\begin{equation}
  \dot u_m
  =\frac{P_{m,\sigma}H\dot v_{m,\sigma}}
  {\|Hv_{m,\sigma}\|_2}.
\end{equation}
Combining this identity with
Eq.~\eqref{eq:app-mean-velocity} gives
\begin{equation}
  \kappa_m(\sigma)
  =\frac{
    \|P_{m,\sigma}H\dot\mu_{m,\sigma}\|_2
  }{
    \sigma\|Hv_{m,\sigma}\|_2
  }.
  \label{eq:app-posterior-turning}
\end{equation}

To relate posterior-mean variation to
$\mathcal I_m(\sigma)$, define the posterior-path score
\begin{equation}
  \ell_{m,\sigma}(x_0)
  =\frac{d}{d\sigma}\log\pi_{m,\sigma}(x_0).
\end{equation}
Normalization gives
$\E_{\pi_{m,\sigma}}[\ell_{m,\sigma}]=0$.
Differentiating the posterior mean under the integral
therefore yields
\begin{equation}
  \dot\mu_{m,\sigma}
  =\E_{\pi_{m,\sigma}}\!\left[
    (x_0-\mu_{m,\sigma})\ell_{m,\sigma}(x_0)
  \right].
  \label{eq:app-posterior-mean-score}
\end{equation}
Applying $P_{m,\sigma}H$ and using
Cauchy--Schwarz gives
\begin{equation}
  \begin{aligned}
  &\|P_{m,\sigma}H\dot\mu_{m,\sigma}\|_2^2\\
  &\quad\leq
  \operatorname{tr}\!\left(
    P_{m,\sigma}H\Sigma_{m,\sigma}
    H^\top P_{m,\sigma}
  \right)\mathcal I_m(\sigma).
  \end{aligned}
  \label{eq:app-cauchy}
\end{equation}
Substituting Eq.~\eqref{eq:app-posterior-turning}
and rearranging establishes the pointwise bound:
\begin{equation}
  \mathcal I_m(\sigma)\geq
  \frac{
    \sigma^2\|v_m^{\mathrm c}(\sigma)\|_2^2
    \kappa_m(\sigma)^2
  }{
    \operatorname{tr}\!\left(
      P_{m,\sigma}H\Sigma_{m,\sigma}
      H^\top P_{m,\sigma}
    \right)
  }.
  \label{eq:curvature-fisher-bound}
\end{equation}

\paragraph{Average bound.}
Suppose there exist constants $c_v>0$ and
$0<C_\Sigma<\infty$ such that, uniformly over the
trajectories and noise interval under consideration,
\begin{equation}
  \begin{aligned}
  \sigma\|Hv_{m,\sigma}\|_2&\geq c_v,\\
  \operatorname{tr}\!\left(
    P_{m,\sigma}H\Sigma_{m,\sigma}
    H^\top P_{m,\sigma}
  \right)&\leq C_\Sigma.
  \end{aligned}
  \label{eq:app-nondegeneracy}
\end{equation}
These conditions keep the noise-scaled centered velocity
away from zero and uniformly bound the projected
posterior covariance. Taking expectations in
Eq.~\eqref{eq:curvature-fisher-bound} and applying
Jensen's inequality gives
\begin{equation}
  \begin{aligned}
  \E_m[\mathcal I_m(\sigma)]
  &\geq\frac{c_v^2}{C_\Sigma}
  \E_m[\kappa_m(\sigma)^2]\\
  &\geq\frac{c_v^2}{C_\Sigma}
  \bigl(\E_m[\kappa_m(\sigma)]\bigr)^2\\
  &=C_I\bar\kappa(\sigma)^2,
  \end{aligned}
\end{equation}
where $C_I=c_v^2/C_\Sigma$ and
$\bar\kappa(\sigma)=\E_m[\kappa_m(\sigma)]$.
This proves Eq.~\eqref{eq:mean-curvature-fisher}.

\paragraph{Connection to the empirical statistic.}
The analysis assumes the exact conditional velocity
field and its continuous flow trajectories. In practice,
we estimate turning from discrete states along numerical
trajectories of the CFG-guided teacher field, which need
not coincide with the exact conditional field.
The proposition therefore motivates the empirical
turning statistic without establishing the exact solution.

\section{Diagnostic Protocols}

\subsection{Representations for temporal affinity}
\label{app:representation-probes}

\paragraph{Protocol.}
We compare VAE and V-JEPA temporal affinities $A_i$ on 32 frozen student rollouts with paired targets. Both spaces use
Eq.~\eqref{eq:motion-score} with lags $\{1,2,4\}$ and weights
$(0.2,0.3,0.5)$.

As an independent reference, we use CoTracker3~\cite{cotracker3}
to track selected target-video points in each generated video,
initialized at corresponding first-frame coordinates. We normalize
2-D point displacements by image width and height and compare
them at matched relative times. Using the same temporal lags,
weights, and similarity formula as for $A_i$, we obtain a
trajectory-agreement score $T_i$ based on point motion rather
than feature changes. Spearman $\rho$ compares the
rankings of $A_i$ and $T_i$. Generator quartile recall measures how many
of the eight lowest-$T_i$ rollouts are recovered by the eight lowest-$A_i$
rollouts.

\paragraph{Results.}
Figure~\ref{fig:representation-probes}(a) plots temporal-affinity and
trajectory-agreement ranks for all 32 pairs. VAE/V-JEPA affinities attain
Spearman correlations of $0.162/0.374$, Kendall correlations of
$0.117/0.246$, and quartile recalls of $25\%/50\%$.

\subsection{Motion magnitude and critic fitting difficulty}
\label{app:critic-predictability}

\paragraph{Measurement protocol.}
We evaluate 32 frozen rollouts with the step-1500 critic at
$\sigma\in\{0.5,0.8,0.92,0.97\}$.
FM loss is the mean squared velocity error over future latent blocks,
averaged over three independent re-noising draws per rollout and
noise level.
We measure latent motion magnitude by averaging RMS latent differences
over time and lags $\{1,2,4,8\}$

\paragraph{Offline representation probes.}
We fit two-hidden-layer MLPs to predict log-FM loss from VAE and
V-JEPA temporal descriptors, using four-fold cross-validation grouped
by rollout. Standardization and PCA are fitted on training folds only.
Out-of-fold Spearman correlations are computed within each noise level
and averaged over all evaluated levels. VAE descriptors
obtain $0.460\pm0.048$, compared with $-0.254\pm0.125$ for V-JEPA temporal
features. Their hardest-quartile recalls are $36.7\pm5.9\%$ and
$12.5\pm5.7\%$, respectively (chance: $25\%$).

\subsection{Difficulty ranking by the trained predictor}
\label{app:head-ranking}

We evaluate the online-trained predictor from the frozen step-2000
\method checkpoint on 32 fresh rollouts generated with training-time
random denoising exits.
At each noise level, Spearman $\rho$ compares predicted fitting
difficulty $\widehat r_i$ with critic FM loss averaged over three
independent re-noising draws per rollout.
Table~\ref{tab:trained-head-ranking} reports correlations of
$0.758$--$0.952$, showing that the trained predictor captures
rollout-dependent fitting-difficulty rankings.

\begin{table}[tb]
\centering
\caption{\textbf{Difficulty ranking by the trained predictor.}
Predictor $\rho$ is the Spearman correlation between predicted fitting
difficulty $\widehat r_i$ and measured critic FM loss on 32 fresh rollouts.
FM losses average three independent noise draws at each noise level.}
\label{tab:trained-head-ranking}
\small
\setlength{\tabcolsep}{5pt}
\begin{tabular}{lccc}
  \toprule
  Noise $\sigma$ & $0.5$ & $0.8$ & $0.97$ \\
  \midrule
  Predictor $\rho$ & $0.758$ & $0.898$ & $0.952$ \\
  \bottomrule
\end{tabular}
\end{table}

\section{Objective and Optimization Details}

\subsection{Dynamics-guided fake-score tracking}
\label{app:critic-details}

For detached student latents $X_i$, we form RMS-difference sequences
$q_{i,j}^{(\ell)}=\operatorname{RMS}(X_{i,j+\ell}-X_{i,j})$
at lags $\ell\in\{1,2,4,8\}$. The fifth sequence measures RMS latent differences relative to the first block. Each sequence $q$ is summarized by
$S(q)=[\operatorname{mean}(q),\operatorname{std}(q),\max(q),q_{\mathrm{last}}]$.
We concatenate the summaries in lag order $1,2,4,8$, followed by the
first-block displacement summary, into a 20-dimensional vector.
Transforming each component $x$ to $\log(1+x)$ gives the descriptor $d_i$.
The two-hidden-layer MLP
receives $d_i$ together with the fixed Fourier noise-level embedding
$e(t_i)$ to predict fitting difficulty $\widehat r_i$.

Using the per-rollout mean squared FM loss over future latent blocks,
the noise-bin baseline and predictor supervision are updated as
\begin{equation}
\begin{aligned}
  b_k&\leftarrow\rho_b b_k+(1-\rho_b)
  \underset{i:k(t_i)=k}{\operatorname{mean}}
  \log(\ell_i^{\mathrm{FM}}+\epsilon_c),\\
  r_i&=\sg[\log(\ell_i^{\mathrm{FM}}+\epsilon_c)-b_{k(t_i)}],\\
  \mathcal L_{\mathrm{pred}}
  &=\frac1B\sum_{i=1}^B\operatorname{Huber}(\widehat r_i,r_i).
\end{aligned}
\label{eq:difficulty-target}
\end{equation}
Here $\rho_b$ is the EMA decay and $\epsilon_c>0$ stabilizes the logarithm.
Bin statistics are aggregated across workers; bins without samples retain
their previous baseline.

Table~\ref{tab:method-hparams} collects the default hyperparameters used.

\begin{table}[t]
  \centering
  \caption{Default hyperparameters for re-noise sampling and fake-score tracking.}
  \label{tab:method-hparams}
  \small
  \setlength{\tabcolsep}{4pt}
  \begin{tabular}{ll}
    \toprule
    Quantity & Choice \\
    \midrule
    Latent measurement / spatial size & 21 blocks \\
    V-JEPA input / output blocks & 16 frames / 8 blocks \\
    Temporal-affinity lags / weights & $\{1,2,4\}$ / $(0.2,0.3,0.5)$ \\
    Teacher-prior probability floor $\alpha$ & $0.3$ \\
    Turning-rate offset $\epsilon_\kappa$ & $0.1\,\mathrm{median}(\widetilde\kappa)$ \\
    Critic timestep sampling & Same rule as generator \\
    Latent descriptor lags & $\{1,2,4,8\}$ \\
    Descriptor dimension & $20$ \\
    Critic residual ema decay $\rho_{b}$ & $0.99$ \\
    Difficulty predictor & Two-hidden-layer MLP \\
    Critic-weighting strength $\beta$ & $0.5$ \\
    Predictor loss weight $\lambda_{\mathrm{pred}}$ & $1.0$ \\
    \bottomrule
  \end{tabular}
\end{table}

\section{Benchmark Protocols and Metrics}
\label{sec:benchmark-context}

\paragraph{R-Bench.}
We evaluate on 400 image--prompt pairs from the embodiment split of R-Bench~\cite{rovidx}, spanning four robot embodiments(dual-arm, humanoid, single-arm, and
quadruped). GPT-5 assigns each video a 1--5 task-completion rating $r_v$ from a grid of six uniformly sampled frames. We normalize scores as $(r_v-1)/4$ and compute Overall by averaging first within each embodiment, then equally across the four embodiments

\paragraph{PAI-Bench-G.}
We evaluate the 174-condition robot subset of PAI-Bench-G~\cite{zhou2026pai}, comprising
913 binary questions. Eight frames are sampled uniformly and judged with Qwen3-VL-235B
using the official binary-VQA templates.
Domain averages per-video accuracies, while Space, Physics, and Time micro-average question accuracies within each axis, assessing spatial interactions, physical plausibility, and temporal progression, respectively.

PAI Quality is the mean of eight normalized dimensions: subject consistency,
background consistency, motion smoothness, aesthetic quality, imaging quality, overall
consistency, and I2V subject/background fidelity. Table~\ref{tab:pai_quality_breakdown} gives the breakdown.

\paragraph{EZS-Bench.}
EZS-Bench~\cite{abotphysworld} contains 196 unseen robot--task--scene conditions and 2,272
binary questions. We use the same per-video and per-axis aggregation as PAI-Bench-G.   Under the official two-model evaluation protocol, Qwen3-VL-32B-Thinking
generates the questions and physics checklists, while
Qwen2.5-VL-72B-Instruct scores the videos. Folloing PhysisForcing~\cite{physisforcing}, we use GPT-5 to compress each original description into a 40--60-word action-centric prompt, using the first-frame image as context. The original questions and checklists remain unchanged for scoring.

\begin{table}[t]
  \centering
  \caption{\textbf{PAI-Bench-G Quality breakdown.}
  Quality is the unweighted mean of the eight official metrics. SC: subject consistency;
  BC: background consistency; MS: motion smoothness; AQ: aesthetic quality; IQ: imaging
  quality; OC: overall consistency; IS/IB: image-to-video subject/background fidelity.
  All values are percentages; higher is better. Column-wise best values,
  including ties, are bold and shaded blue.}
  \label{tab:pai_quality_breakdown}

  \footnotesize
  \setlength{\tabcolsep}{3.5pt}
  \renewcommand{\arraystretch}{1.08}
  \begin{tabular*}{\columnwidth}{@{\extracolsep{\fill}}lccccc@{}}
    \toprule
    Method & Quality & SC & BC & MS & AQ \\
    \midrule
    PF-Wan 14B
    & 76.9 & 90.3 & 90.6 & 98.9 & \cellcolor{intdrei!70}\textbf{50.1} \\
    Wan2.1-Fun 1.3B
    & 75.3 & 87.8 & 91.6 & 98.6 & 47.6 \\
    \midrule
    Base DMD
    & 78.1
    & 95.3
    & 91.9 & 99.5 & 48.6 \\
    \textbf{\method}
    & \cellcolor{intdrei!70}\textbf{78.1}
    & \cellcolor{intdrei!70}\textbf{95.3}
    & \cellcolor{intdrei!70}\textbf{92.2}
    & \cellcolor{intdrei!70}\textbf{99.6}
    & 48.5 \\
    \bottomrule
  \end{tabular*}

  \vspace{2pt}
  \begin{tabular*}{\columnwidth}{@{\extracolsep{\fill}}lcccc@{}}
    \toprule
    Method & IQ & OC & IS & IB \\
    \midrule
    PF-Wan 14B
    & 71.6 & \cellcolor{intdrei!70}\textbf{19.9} & 96.2 & 97.5 \\
    Wan2.1-Fun 1.3B
    & 67.2 & 19.4 & 94.4 & 96.0 \\
    \midrule
    Base DMD
    & \cellcolor{intdrei!70}\textbf{73.2}
    & 19.5 & 97.8 & 98.9dymd \\
    \textbf{\method}
    & 73.1 & 19.6
    & \cellcolor{intdrei!70}\textbf{97.9}
    & \cellcolor{intdrei!70}\textbf{98.9} \\
    \bottomrule
  \end{tabular*}
\end{table}

\begin{figure*}[t!]
  \centering
  \includegraphics[width=.90\textwidth]{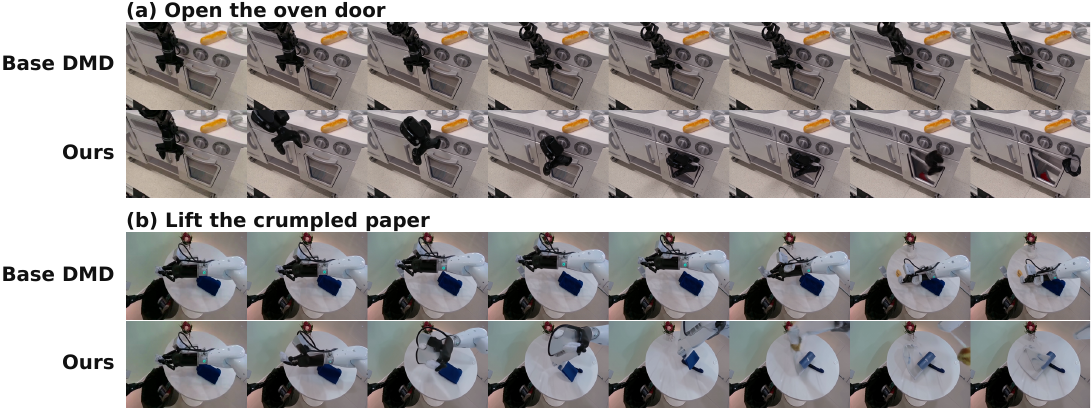}
  \caption{\textbf{Additional interaction-dynamics comparisons at four NFE.}
  Eight uniformly spaced frames are shown for each of two PAI-Bench-G cases,
  with prompts, seeds, and step-1500 checkpoints matched to the main qualitative
  results. Base DMD often preserves the initial configuration or incompletely
  executes the action, whereas \method produces clearer object-directed motion.}
  \label{fig:qualitative-supp}
\end{figure*}
\begin{table}[t]
  \centering
  \caption{\textbf{Video-native versus framewise temporal-affinity features.}
  We compare V-JEPA~2.1-L and DINOv2-L for affinity-conditioned sampling under the same
  training budget. Bold with blue shading marks the
  better result.}
  \label{tab:encoder_ablation}
  \footnotesize
  \setlength{\tabcolsep}{3.5pt}
  \renewcommand{\arraystretch}{1.10}
  \begin{tabular*}{\columnwidth}{@{\extracolsep{\fill}}lcccc@{}}
    \toprule
    \multirow{2}{*}{Encoder}
    & \multicolumn{1}{c}{R-Bench}
    & \multicolumn{2}{c}{PAI-Bench-G}
    & \multicolumn{1}{c}{EZS-Bench} \\
    \cmidrule(lr){2-2}\cmidrule(lr){3-4}\cmidrule(l){5-5}
    & TAC & Domain & Quality & Domain \\
    \midrule
    V-JEPA 2.1-L
    & \cellcolor{intdrei!70}\textbf{43.5}
    & \cellcolor{intdrei!70}\textbf{78.5}
    & \cellcolor{intdrei!70}\textbf{78.1}
    & \cellcolor{intdrei!70}\textbf{75.1} \\
    DINOv2-L & 41.8 & 76.7 & 77.8 & 75.0\\
    \bottomrule
  \end{tabular*}
\end{table}

\subsection{Downstream action-planning protocol}
\label{app:downstream}

\paragraph{Frozen world-model features.}
Following WorldArena's action-planner interface~\cite{worldarena}, we use a VPP inverse-dynamics head~\cite{vpp} to predict actions from frozen student features. At each planning step, we condition the student on the current image and instruction and extract intermediate DiT features at the initial, highest-noise timestep. Feature extraction requires only a single partial forward pass, without completing the denoiser or decoding an RGB rollout.

\paragraph{VPP training.}
We train a separate VPP head for each backbone and task using the
50 demonstrations in \texttt{aloha-agilex\_clean\_50}.
Each training sample pairs an observation with a 45-step sequence
of 14-dimensional joint actions.
The input features span 21 latent temporal blocks, corresponding
to an 81-frame horizon.
VPP uses a six-layer Video Former with eight attention heads and
399 latent queries, followed by an action-denoising head.

We train for 200 epochs with AdamW, using a learning rate of $10^{-4}$,
weight decay of $0.05$, and $(\beta_1,\beta_2)=(0.9,0.9)$,
with eight processes and a batch size of eight per process.
For Adjust Bottle, we evaluate the Base DMD and \method heads at
steps 13904 and 15010, respectively; for Click Bell, the corresponding
steps are 4509 and 5319.

\paragraph{Closed-loop evaluation.}
VPP predicts 45-action chunks with ten action-denoising steps.
Actions are executed in RoboTwin~2.0~\cite{robotwin2}, with updated
observations conditioning subsequent predictions. Each backbone is evaluated over 100 episodes per task.

\section{Additional Ablation Studies}
\label{app:additional-ablations}

\subsection{Video-native versus framewise temporal-affinity features}
\label{app:encoder-ablation}

Table~\ref{tab:encoder_ablation} compares video-native V-JEPA~2.1-L with
framewise DINOv2-L~\cite{oquab2023dinov2} for temporal affinity--conditioned
re-noise sampling. For DINOv2, we
encode 16 uniformly sampled frames into $24\!\times\!32$ patch grids,
apply spatial pooling, and average adjacent frame pairs into eight
temporal blocks.
Relative to DINOv2, V-JEPA improves TAC, PAI Domain, Quality, and EZS
Domain by $1.7$, $1.8$, $0.3$, and $0.1$ points, respectively,
supporting video-native features for target-relative interaction assessment.

\subsection{Teacher-prior probability floor}
\label{app:alpha-ablation}

Table~\ref{tab:alpha_sensitivity} varies
$\alpha\in\{0.1,0.3,0.5\}$ with tracking disabled.
Increasing $\alpha$ raises the minimum teacher-prior sampling
probability, allocating more samples under the teacher prior schedule.
The weaker results at $\alpha=0.1$ suggest that these intervals
receive insufficient supervision: increasing $\alpha$ to $0.3$
improves TAC, PAI Domain, and EZS Domain by $4.3$, $1.1$,
and $1.9$ points, respectively. Further increasing $\alpha$ to $0.5$ yields comparable
performance, with TAC increasing by $0.3$ points and both
Domain scores decreasing by $0.3$ points. We therefore use the moderate floor $\alpha=0.3$, which provides sufficient
teacher-prior coverage

\begin{table}[!htb]
  \centering
  \caption{\textbf{Sensitivity to the teacher-prior probability floor $\alpha$.}
  Only $\alpha$ varies, with tracking disabled.
  The $\alpha=0.3$ row repeats the sampling-only result from
  Table~\ref{tab:component_ablation}. Column-wise best values,
  including ties, are bold and shaded blue.}
  \label{tab:alpha_sensitivity}
  \footnotesize
  \setlength{\tabcolsep}{3.5pt}
  \renewcommand{\arraystretch}{1.10}
  \begin{tabular*}{\columnwidth}{@{\extracolsep{\fill}}lcccc@{}}
    \toprule
    \multirow{2}{*}{$\alpha$}
    & \multicolumn{1}{c}{R-Bench}
    & \multicolumn{2}{c}{PAI-Bench-G}
    & \multicolumn{1}{c}{EZS-Bench} \\
    \cmidrule(lr){2-2}\cmidrule(lr){3-4}\cmidrule(l){5-5}
    & TAC & Domain & Quality & Domain \\
    \midrule
    $0.1$ & 39.2 & 77.4 & 78.1 & 73.2 \\
    $0.3$ & 43.5
          & \cellcolor{intdrei!70}\textbf{78.5}
          & \cellcolor{intdrei!70}\textbf{78.1}
          & \cellcolor{intdrei!70}\textbf{75.1} \\
    $0.5$ & \cellcolor{intdrei!70}\textbf{43.8}
          & 78.2 & 78.1 & 74.8 \\
    \bottomrule
  \end{tabular*}
\end{table}

\end{document}